\documentclass{article}

\usepackage{microtype}
\usepackage{graphicx}
\usepackage{subcaption}
\usepackage{booktabs}
\usepackage{hyperref}

\usepackage[preprint]{icml2026}

\usepackage{amsmath}
\usepackage{amssymb}
\usepackage{mathtools}
\usepackage{amsthm}
\usepackage[capitalize,noabbrev]{cleveref}
\crefname{appendix}{Appendix}{Appendices}
\Crefname{appendix}{Appendix}{Appendices}

\theoremstyle{plain}
\newtheorem{theorem}{Theorem}[section]

\newtheorem{lemma}[theorem]{Lemma}

\theoremstyle{definition}

\theoremstyle{remark}

\newcommand{\R}{\mathbb{R}}
\newcommand{\sig}{\sigma}
\newcommand{\norm}[1]{\lVert #1 \rVert}

\icmltitlerunning{Two Missing Primitives of Softmax Attention}

\begin{document}

\twocolumn[
  \icmltitle{Abstention and Noise Filtering: \\ Two Missing Primitives
  of Softmax Attention}

  \begin{icmlauthorlist}
    \icmlauthor{Richard Zhe Wang}{sjfu}
  \end{icmlauthorlist}

  \icmlaffiliation{sjfu}{St.\ John Fisher University, Rochester, New York, USA}

  \icmlcorrespondingauthor{Richard Zhe Wang}{zwang@sjfc.edu,
  rzwang.research@gmail.com}

  \icmlkeywords{attention, value gating, attention sinks, superposition,
  language models}

  \vskip 0.3in
]

\printAffiliationsAndNotice{}

\begin{abstract}
Softmax attention has two structural gaps. A head cannot abstain,
because its weights sum to one, so it outputs something even when
nothing is relevant. Nor can it filter what it reads,
because its output is a weighted average of value vectors, passing
interference as faithfully as signal. We call these
missing primitives \emph{abstention} and \emph{noise filtering}.
Recent studies report that gating the value pathway improves
pretraining but attribute the gain to different causes. We show that
a value gate partly supplies both primitives, which unifies the
reported causes as views of one gain. We give each primitive its own
mechanism in matched models of 10M
to 350M parameters and measure what each contributes. The gain from gating is
almost entirely abstention at 10M, whereas by 350M filtering
contributes as much as abstention, so what a study observes depends
on its scale. The two benefits are
largely additive, with a small overlap. A gate determined
by each value alone leaves the attention sink in place,
whereas a query-controlled mechanism removes it. Injecting interference into the value reads shows that abstention and
filtering protect against it in distinguishable ways. The same patterns
appear in pretrained models up to 20B parameters.
\end{abstract}

% =====================================================================
\section{Introduction}\label{sec:intro}

Softmax attention has two structural gaps. The first gap is that a head
cannot abstain. Its attention weights sum to one, so a head whose
query finds nothing relevant in the context must still place its
full weight somewhere. Models cope by placing the unwanted weight on
a fixed position, usually the first token, known as the attention
sink \citep{xiao2024streamingllm, gu2025sink, barbero2025first},
whose value vector they learn to make small. The second gap is that a
head cannot filter the noise in the values it
reads. Its output is a weighted average of the value vectors it attends
to, and a weighted average passes interference as faithfully as the real
signal. Because of
superposition in the residual stream \citep{elhage2022toy}, a value
vector carries interference from features that are present but
irrelevant to the head. Nothing in the value pathway of standard
softmax attention can remove this interference. We call these two missing
primitives \emph{abstention} and \emph{noise filtering}. In this
paper we argue that these two primitives are what a learned gate on
the value pathway of attention supplies, and we show how their
contributions change with scale.

Both gaps have been partly filled by recent studies, though not
identified as such. Gated attention \citep{qiu2025gated},
value-state gated attention \citep{bu2025vga}, and GLU attention
\citep{wang2025glu} insert a learned gate into the output or value
pathway of attention and
report consistent gains in pretraining loss. \Citet{fesser2026unifying}
relate such gates to attention sinks. While these studies agree that
gates help, they differ in the cause to
which they attribute the gain. The causes offered include added
nonlinearity and sparsity, the elimination of attention sinks, the
mitigation of extreme-token phenomena, and the ability of a head to
perform no operation. No prior studies have systematically
investigated these individual solutions within the unifying framework
of abstention and noise filtering. This paper attempts to fill this
gap.

Value gating supplies part of each missing primitive. It can stand in
for \emph{abstention} in part, because attenuating every value a head
attends to shrinks the head's output, which is what an abstaining
head would produce. It also provides \emph{noise filtering} when the
multiplier is computed
from the value itself, as in \citet{bu2025vga} and \citet{wang2025glu}.
A value scaled by a nonlinear function of its own content can be
suppressed when that content is interference.
A gate computed from the query and applied to the head's output, as
in \citet{qiu2025gated}, can abstain but cannot filter individual
values. A short lemma (\cref{sec:valueslot}) shows that this decomposition holds as an identity. Because one gate supplies both, a study that measures its benefit
observes a mixture of two effects, which helps explain why the studies
above reach different conclusions.

To separate the two primitives, we give each its own mechanism
(\cref{sec:variants}) and train the resulting variants in matched
models of 10M to 350M parameters with paired seeds. For abstention we
use a learned per-head sink logit, as in gpt-oss
\citep{openai2025gptoss}. It adds to every softmax a phantom key, an
extra key that no token occupies, with a learnable logit and a value
fixed at zero. For filtering we use two gate forms, a norm gate that
thresholds the
norm of the value vector and a projection gate that applies a learned
linear function of its content. Combined variants pair each gate
with the sink logit.
\Cref{sec:attribution} asks which primitive a gate's benefit comes from
at each scale and whether the two benefits are separate, and
\cref{sec:fingerprint} asks what the filter removes.

Our contributions are four-fold. (i) We propose two missing primitives
of softmax attention, abstention
and noise filtering, as a unifying framework for value gating. We show
that a scalar gate on the values decomposes into the two, and we
construct mechanisms that supply each on its own (\cref{sec:variants}).
(ii) We show that the composition of the benefit changes with scale,
from abstention at 10M to an equal share of filtering at 350M. We
explain the change
through the baseline's attention sinks (\cref{sec:attribution}). (iii)
We show that the two benefits are
separate mechanisms, because they are largely additive, with a
small overlap (\cref{sec:additivity}).
(iv) We introduce injection experiments that perturb the value reads at
evaluation time and show with them that abstention and filtering protect
against injected interference in distinguishable ways
(\cref{sec:fingerprint}).

% =====================================================================
\section{Related Work}\label{sec:related}
% Literature sweep refreshed 2026-09-18; every entry verified on arXiv.

\paragraph{Value and output gating.} Gates on the output or the values
of attention have a longer history as components of efficient attention
\citep{hua2022flash, shleifer2021normformer} and as a remedy for
outliers \citep{bondarenko2023quantizable}. The recent interest in them
follows the gated attention of \citet{qiu2025gated}. Their gate is a
per-head sigmoid, computed from the query token's representation and
applied to the output of attention. They attribute the resulting gains
to the added
nonlinearity and sparsity, and they observe that attention sinks
disappear. Value-state gated attention \citep{bu2025vga} instead
places the gate on each value and computes it from the value itself, as
a sigmoid of a learned linear
projection. They attribute the benefit to the mitigation of
extreme-token phenomena, namely attention sinks and the massive
activations that accompany them. GLU attention \citep{wang2025glu} reports gains from a gated linear unit
applied to the values. \Citet{fesser2026unifying} distinguish sinks that let a head perform no
operation from sinks that broadcast information. They hold that gating
addresses the first kind, whereas extra learned tokens, called
registers, address the second. These studies agree that
gating helps and differ in how they explain it. In the framework of this paper, each explanation describes part of what
a value gate supplies, and no study separates the two primitives
experimentally. The design of \citet{qiu2025gated} has since been
deployed in production models \citep{su2026sinksurvey}, and further
variants gate at both the value and output stages \citep{shen2026gsa},
combine gate forms \citep{zhou2026hyga}, replace the gate by an affine
scale \citep{bae2026affine}, or include per-head gating in a small-model
training recipe \citep{grigorev2026imu1}, while
\citet{nguyen2026statistical} give a theoretical account in terms of
sample efficiency.

\paragraph{Abstention and attention sinks.} \Citet{miller2023offbyone}
notes that softmax attention cannot output
nothing and proposes adding one to the softmax denominator so that a
head can. Trained models compensate for this inability by forming
attention sinks \citep{xiao2024streamingllm}, and the production
model gpt-oss gives each head a learnable sink logit for the same
purpose \citep{openai2025gptoss}. Sinks arise
as biases in the keys, and their strength grows with scale
\citep{gu2025sink}. They serve as an approximate no-operation that
prevents a head from
mixing too much information across positions \citep{barbero2025first}.
\Citet{sukenik2026sink} prove that a sink pattern is equivalent to a
zero output and that sinks act as attention switches that prevent
oversmoothing. Sinks are accompanied by massive activations and crushed value norms
\citep{sun2024massive, guo2024activedormant} and have been interpreted
as reference frames \citep{ruscio2025sinking}. Further studies show that
a learnable bias makes them disappear \citep{qiu2026unified}, that a
sink logit acts as implicit per-head gating \citep{fu2026sinkmoe}, that
sinks are necessary in some tasks \citep{ranmilo2026necessary}, and that
a null state in the softmax stabilizes attention residuals
\citep{luo2026oasis}. Alternatives to a
sink include rectified softmax \citep{zuhri2025softpick} and register
tokens \citep{darcet2024registers, jiang2025registersfree}, and
\citet{su2026sinksurvey} survey the area. Most closely related to our analysis of sinks in the trained models,
\citet{sun2026spike} train a 7B model from scratch and find that a
gate computed from the current token's representation eliminates
sinks.
They conclude that a sink is an implicit gate that the model abandons
once an explicit one is provided. In the terms of this paper, a sink, a
sink logit, and a query-computed
gate are all abstention mechanisms. \paragraph{Per-key scalar gates on routing.}
A different family of gates acts on the attention weights, which route
each query to its keys, rather than on the values. The weights therefore
still sum to one after the gate is applied. Examples
include the learned per-key decay of the forgetting transformer
\citep{lin2025fox} and the learned mask of selective attention
\citep{leviathan2025selective}. Further members include the
adaptive-pruning and fine-grained variants of the forget gate
\citep{lin2025acp, fix2026} and token pruning that removes keys outright
\citep{rao2021dynamicvit}, and \citet{jamil2026routing} use the same two words for a different
decomposition, the symmetric and skew-symmetric parts of the score
matrix, which concerns queries and keys rather than values. The norm gate of this paper builds on
the observation of
\citet{kobayashi2020norm} that what a position contributes to a head's
output is the product of its attention weight and the norm of its value
vector. 

\paragraph{Superposition, interference, and thresholding.}
The need for noise filtering follows from superposition.
\Citet{elhage2022toy} establish that a model can represent more features
than it has dimensions, a regime known as superposition, in which the
features interfere with one another. The theory of computation in
superposition shows that a linear readout
cannot reduce this interference below a floor, whereas a thresholding nonlinearity, such as a gate, can
\citep{hanni2024cis, adler2024complexity, borobia2026threshold}.
\Citet{adler2025capacity} gives the corresponding capacity argument for
the query-key side. \Citet{prieto2026correlations} show that
interference between correlated features can be constructive while a
nonlinearity is still required to prevent false positives, and
\citet{stevinson2025adversarial} show that adversarial attacks exploit
interference between superposed features. Two other recent works modify the value or query pathway on grounds
complementary to ours, one by replacing
deep-layer values with context-free lookups \citep{he2026bov} and one by
adding a nonlinearity to the query projection
\citep{karbevski2026nonlinear}. 

% \paragraph{Methodology.}
% \citet{narang2021modifications} find that most transformer
% modifications fail to transfer across implementations, and
% \citet{zhao2026modifications} update that finding through 2026 with
% a measured noise floor, that is, the size of the loss differences
% that random seeds alone produce. Because our effect sizes lie in the
% regime
% they describe, every comparison in this paper is a paired per-seed
% difference.

% =====================================================================
\section{Mechanisms and Model Variants}\label{sec:variants}

\subsection{Setup and the Two Primitives}\label{sec:setup}

Consider a single causal attention head with head dimension $d$.
For a query position $i$, the head forms the logits
$s_{ij} = q_i^\top k_j / \sqrt{d}$ over the key positions $j \le i$
and converts them into routing weights
\begin{equation}
  a_{ij} = \frac{\exp(s_{ij})}{\sum_{k \le i} \exp(s_{ik})}.
  \label{eq:softmax}
\end{equation}
It then returns the aggregate $y_i = \sum_{j \le i} a_{ij}\, v_j$,
where $q_i = W_Q x_i$, $k_j = W_K x_j$, and $v_j = W_V x_j \in \R^d$
are the query, key, and value vectors computed by learned
projections from the residual-stream inputs $x_i$ and $x_j$. Every
quantity here depends on the parameters and on the input, and
$a_{ij}$ depends on $q_i$ and on every key $k_1, \ldots, k_i$. We
refer to $a_{ij}$ as
\emph{routing} and $v_j$ as a \emph{read}, because $v_j$ is the content 
that the head obtains when it attends to position $j$.

Two key properties follow from this standard setup. First, the
routing weights of each row sum to one, so the head \emph{must}
distribute its full attention mass of one even when no prior position
is worth reading. This inability to attend to nothing is a known
limitation of softmax attention \citep{miller2023offbyone,
xiao2024streamingllm}. 

Second, the output $y_i$ is a linear function
of the reads $v_j$. Hence, a read that consists mostly of interference
is aggregated with the same fidelity as a read that consists mostly
of signal. Nothing in the aggregation can therefore suppress a read. In these
terms, abstention is the ability of a head to place routing mass on
nothing, and noise filtering is the ability to attenuate a read on
the basis of its content. Abstention is determined on the routing side, by
whether the query finds
anything in the context worth attending to. Noise filtering is determined
on the value side, by whether the content of a read should be passed
into the residual stream. The distinction lies in what determines the
suppression, not in its effect on the output.

\subsection{Gating in the Value Slot}\label{sec:valueslot}

We say that a gate acts in the \emph{value slot} when it multiplies
the reads and in the \emph{routing slot} when it modifies the
routing weights $a_{ij}$. A value gate multiplies each read by a
scalar $g_j \in (0, 1)$ computed from the read $v_j$ alone, so that
the routing of \cref{eq:softmax} is unchanged and the output becomes
\begin{equation}
  y_i' = \sum_{j \le i} a_{ij}\, g_j\, v_j.
  \label{eq:valuegate}
\end{equation}

\begin{lemma}[Decomposition of a value-slot gate]\label{lem:scalar}
Let $a_i = (a_{i1}, \ldots, a_{iT})$ be routing weights with
$\sum_j a_{ij} = 1$, let $g = (g_1, \ldots, g_T)$ be non-negative
scalars, one per key position, and let $Z_i = \sum_j a_{ij}\, g_j$. If
$Z_i > 0$,
then
\begin{equation}
  \sum_j a_{ij}\, g_j\, v_j \;=\; Z_i \sum_j \tilde a_{ij}\, v_j,
  \qquad
  \tilde a_{ij} = \frac{a_{ij}\, g_j}{Z_i},
  \label{eq:lemma}
\end{equation}
where $\tilde a_i = (\tilde a_{i1}, \ldots, \tilde a_{iT})$ is again a
probability vector.
\end{lemma}

\begin{proof}
Since $Z_i > 0$, each term can be multiplied and divided by $Z_i$,
\begin{equation*}
  \sum_j a_{ij}\, g_j\, v_j
  = Z_i \sum_j \frac{a_{ij}\, g_j}{Z_i}\, v_j
  = Z_i \sum_j \tilde a_{ij}\, v_j,
\end{equation*}
and $\sum_j \tilde a_{ij} = \sum_j a_{ij}\, g_j / Z_i = Z_i / Z_i = 1$.
\end{proof}

The lemma says that applying a gate $g_j$ in the
value slot equals applying the identical per-key scalars $g_j$ in the
routing slot and then scaling the head output by $Z_i$. 
The routing-slot gate redistributes weight among the reads, and $Z_i$
is the fraction of the routing mass that survives the suppression.
$Z_i$ depends on the query through $a_{ij}$, shrinks when the mass of
query $i$ rests on suppressed reads, and reaches zero only if every
attended read has $g_j = 0$. Both factors are determined by the reads,
so a value gate filters in the sense of \cref{sec:setup} even where
its effect on a row is a scalar attenuation, and we count that
attenuation as filtering. A gate confined to the routing slot opens almost fully in
training and yields no benefit (\cref{tab:routing}).

\subsection{The Variants}\label{sec:armlist}

This section describes the eight attention variants of the study.
Together they supply abstention alone, noise filtering alone, both,
or neither, so that the contribution of each primitive can be
measured on its own. \Cref{tab:variants} in \cref{app:training}
lists them with their parameter costs. 

\paragraph{Baseline.} The baseline is standard causal attention as
in \cref{eq:softmax}, with no gate and no abstention mechanism.

\paragraph{Off-by-one and sink logit.} Following
\citet{miller2023offbyone}, the off-by-one variant adds one to the
denominator of \cref{eq:softmax}. This is equivalent to adding an
extra key position that no token occupies, with logit zero and a value of zero. We call this extra position the
\emph{phantom}. Whatever mass the query places on the phantom is
removed from the output, because the phantom contributes nothing to
the aggregate. The sink-logit variant replaces the added constant one by
$\exp(b_h)$, where $b_h$ is a learnable scalar per head initialized to
zero. This is the per-head sink of gpt-oss \citep{openai2025gptoss},
a learned logit with no value vector. The routing weights become
\begin{equation}
  a_{ij} = \frac{\exp(s_{ij})}{\exp(b_h) + \sum_{k \le i} \exp(s_{ik})},
  \label{eq:sinklogit}
\end{equation}
with the off-by-one variant recovered at $b_h = 0$. Each head thereby
learns how readily it abstains, while the phantom value remains zero.

\paragraph{Sink token.} The sink-token variant adds a learned extra
token at the start of every sequence \citep{xiao2024streamingllm}.
Its key and its value are both learned, so the mass it absorbs still delivers its learned value to the output. In particular, any residual content in
its value leaks into every output that attends to it. It is included to test whether the abstained mass must contribute
nothing.

\paragraph{Norm gate.} We propose a new variant that attenuates each read by a
sigmoid of the read's own norm,
\begin{equation}
  g_j = \sig\!\left(\beta\left(\frac{\norm{v_j}}{\sqrt{d}} -
  \tau_h\right)\right),
  \qquad v_j \leftarrow g_j\, v_j,
  \label{eq:normgate}
\end{equation}
with a fixed sharpness $\beta = 8$ and a learnable threshold $\tau_h$
per head. The threshold is initialized to $-4/\beta$, so that $g_j \ge \sig(4)
\approx 0.98$ for every read at initialization. It is excluded from
weight decay, which would otherwise pull it toward filtering regardless
of whether filtering helps. Because the
gate depends on the read only through its norm, it is invariant to
rotations of the value basis. The cost of this invariance is that it
cannot single out
particular directions of the value space. 
This invariance makes the norm gate the appropriate instrument for
establishing whether noise filtering helps a model at all, because
any benefit it produces must come from the magnitude of the read
alone.

\paragraph{Projection gate.} The projection-gate variant attenuates each
read by a sigmoid of a learned linear function of the read,
\begin{equation}
  g_j = \sig\!\left(w_h^\top v_j + b_h\right),
  \qquad v_j \leftarrow g_j\, v_j,
  \label{eq:projgate}
\end{equation}
which is the form of value-state gated attention \citep{bu2025vga}. We initialize it with $w_h = 0$ and $b_h = 4$, so that it starts as the
same near-identity as the norm gate, and we exclude both parameters from
weight decay for the same reason.
Unlike the norm gate, the projection gate is
direction-selective, meaning that it can pass or block a read
according to the direction in which the read points in the value
space. The two gate forms together let us ask whether the quantity that
determines the gate value affects what a filter can and cannot remove.

\paragraph{Two combined variants.} The variant \emph{sink+norm} applies
the gpt-oss sink logit
\citep{openai2025gptoss} and the norm gate together. The variant
\emph{sink+proj} applies the same sink logit and the projection gate
together. When abstention is already supplied by the phantom key in the
sink logit,
any remaining benefit of a gate must then be attributed to noise filtering.

% =====================================================================
\section{The Composition of the Benefit across Scale}\label{sec:attribution}
% Numbers: paper/tables/loss.tex (tables.py), figures.py fig_trend,
% autopsy_m350.log, mech_m350.log.

This section asks which primitive a gate's benefit comes from, how the
answer changes with scale, and whether the two benefits are separate.

\subsection{Experimental Setup}\label{sec:setup4}

We train every variant from scratch on the FineWeb-Edu web-text corpus
\citep{penedo2024fineweb} at four model sizes. We call the sizes tiers
and label them by approximate non-embedding parameter count (10M, 50M,
124M, and 350M; \cref{app:training}). All variants at a tier share
everything, including the data order,
except the attention mechanism.
We train three seeds per variant at every tier. Because seed noise dominates comparisons at this scale
\citep{zhao2026modifications}, every effect is reported as a paired
difference between variants that share a seed, averaged over seeds and given with its standard error.
\Cref{tab:lossmain} reports the paired improvements of the main variants
over the baseline and of the combined variants over the sink logit, and
\cref{tab:loss} in \cref{app:tables} reports every variant. The latter is what a gate adds once a sink logit is in place. 
Improvements are quoted as positive magnitudes. 
Zero-shot evaluations at 124M show the same ordering in perplexity,
with accuracy differences within seed noise (\cref{tab:downstream}).

\begin{figure*}[t]
  \begin{center}
    \centerline{\includegraphics[width=\textwidth]{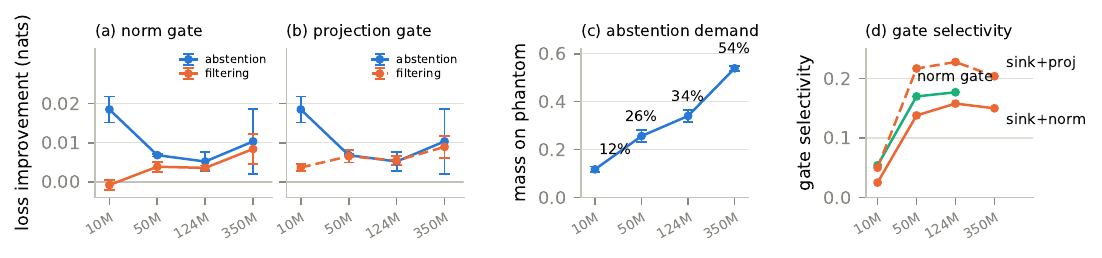}}
    \caption{The two primitives across scale. (a,b) For each gate form, the
    abstention increment (sink logit over baseline) and the filtering
    increment (combined variant over sink logit). The filtering
    increment grows with scale, and the abstention increment shrinks
    from 10M to 124M. Bars are standard errors over three seeds. (c) The average attention mass the sink-logit model places on its
    phantom grows with scale (bars are standard deviations over seeds). This is the phantom's softmax weight, averaged over every query,
    head, layer, and validation sequence at the tier's training context.
    (d) Gate selectivity, the attention-weighted standard deviation of the
    gate values among the reads a query attends to, averaged over
    queries, heads, and layers (seed 0 per tier). It is zero for a gate
    that attenuates every attended read equally.}
    \label{fig:trend}
  \end{center}
  \vskip -0.2in
\end{figure*}

\subsection{Two Opposite Trends with Scale}\label{sec:cross}

\begin{table}[t]
  \caption{Paired validation-loss improvements (nats), mean over
  seeds with its standard error beneath. The lower block is
  the filtering increment, the improvement of each combined variant
  over the sink logit.}
  \label{tab:lossmain}
  \begin{center}
    \begin{footnotesize}
      \setlength{\tabcolsep}{6pt}
      \begin{tabular}{lrrrr}
\toprule
Variant & 10M & 50M & 124M & 350M \\
\midrule
\multicolumn{5}{l}{\emph{Improvement over the baseline}} \\
sink logit & \begin{tabular}[c]{@{}r@{}}$0.0185$\\{\scriptsize$\pm 0.0033$}\end{tabular} & \begin{tabular}[c]{@{}r@{}}$0.0068$\\{\scriptsize$\pm 0.0003$}\end{tabular} & \begin{tabular}[c]{@{}r@{}}$0.0052$\\{\scriptsize$\pm 0.0025$}\end{tabular} & \begin{tabular}[c]{@{}r@{}}$0.0103$\\{\scriptsize$\pm 0.0084$}\end{tabular} \\
norm gate & \begin{tabular}[c]{@{}r@{}}$0.0037$\\{\scriptsize$\pm 0.0026$}\end{tabular} & \begin{tabular}[c]{@{}r@{}}$0.0069$\\{\scriptsize$\pm 0.0008$}\end{tabular} & \begin{tabular}[c]{@{}r@{}}$0.0059$\\{\scriptsize$\pm 0.0012$}\end{tabular} & -- \\
projection gate & \begin{tabular}[c]{@{}r@{}}$0.0113$\\{\scriptsize$\pm 0.0029$}\end{tabular} & \begin{tabular}[c]{@{}r@{}}$0.0094$\\{\scriptsize$\pm 0.0011$}\end{tabular} & \begin{tabular}[c]{@{}r@{}}$0.0091$\\{\scriptsize$\pm 0.0011$}\end{tabular} & -- \\
sink+norm & \begin{tabular}[c]{@{}r@{}}$0.0177$\\{\scriptsize$\pm 0.0037$}\end{tabular} & \begin{tabular}[c]{@{}r@{}}$0.0106$\\{\scriptsize$\pm 0.0015$}\end{tabular} & \begin{tabular}[c]{@{}r@{}}$0.0087$\\{\scriptsize$\pm 0.0021$}\end{tabular} & \begin{tabular}[c]{@{}r@{}}$0.0187$\\{\scriptsize$\pm 0.0060$}\end{tabular} \\
sink+proj & \begin{tabular}[c]{@{}r@{}}$0.0221$\\{\scriptsize$\pm 0.0029$}\end{tabular} & \begin{tabular}[c]{@{}r@{}}$0.0133$\\{\scriptsize$\pm 0.0018$}\end{tabular} & \begin{tabular}[c]{@{}r@{}}$0.0106$\\{\scriptsize$\pm 0.0022$}\end{tabular} & \begin{tabular}[c]{@{}r@{}}$0.0193$\\{\scriptsize$\pm 0.0057$}\end{tabular} \\
\midrule
\multicolumn{5}{l}{\emph{Improvement over the sink logit}} \\
sink+norm & \begin{tabular}[c]{@{}r@{}}$-0.0008$\\{\scriptsize$\pm 0.0013$}\end{tabular} & \begin{tabular}[c]{@{}r@{}}$0.0030$\\{\scriptsize$\pm 0.0009$}\end{tabular} & \begin{tabular}[c]{@{}r@{}}$0.0036$\\{\scriptsize$\pm 0.0006$}\end{tabular} & \begin{tabular}[c]{@{}r@{}}$0.0084$\\{\scriptsize$\pm 0.0039$}\end{tabular} \\
sink+proj & \begin{tabular}[c]{@{}r@{}}$0.0036$\\{\scriptsize$\pm 0.0008$}\end{tabular} & \begin{tabular}[c]{@{}r@{}}$0.0065$\\{\scriptsize$\pm 0.0016$}\end{tabular} & \begin{tabular}[c]{@{}r@{}}$0.0054$\\{\scriptsize$\pm 0.0012$}\end{tabular} & \begin{tabular}[c]{@{}r@{}}$0.0089$\\{\scriptsize$\pm 0.0027$}\end{tabular} \\
\bottomrule
\end{tabular}

    \end{footnotesize}
  \end{center}
  \vskip -0.1in
\end{table}

\Cref{fig:trend}(a,b) plots two main effects across the four tiers,
defined as follows. The abstention increment is the paired
improvement of the sink logit over the baseline, measured in the reduction of the validation loss (nats). 
The filtering increment is the paired improvement of a combined variant over the
sink logit, that is, what a gate adds once an abstention mechanism is already in place. 
The two effects move in opposite directions from 10M to 124M, and
the filtering increment keeps growing at 350M.
\Cref{fig:trend}(a) shows that the abstention increment falls from
0.0185 nats at 10M to 0.0052 at 124M, while the filtering increment
of the norm gate rises from zero at 10M to 0.0084 at 350M.
\Cref{fig:trend}(b) shows the projection gate's filtering increment
similarly rises from zero at 10M to 0.0089 at 350M, and the two
curves cross between 50M and 124M. At 350M the three seeds give abstention increments of 0.0013,
0.0271, and 0.0026, so the mean of 0.0103 is not resolved beyond its
seed spread. In two of the three seeds the filtering increments at 350M exceed
the abstention increment several times over.

\cref{fig:trend}(c) tracks the average attention mass that the sink-logit 
model places on its phantom, which grows from 12\% at 10M to 54\% at 350M, 
indicating that the demand for abstention grows with scale even as
the benefit of supplying it shrinks from 10M to 124M. 

\Cref{fig:trend}(d) shows the gate selectivity of the trained models,
which measures whether a gate discriminates among the reads a query
attends to or attenuates them all alike. For one query $i$, the
attention weights over its keys form one row of the attention matrix,
and the gate assigns each key $j$ in that row a value $g_j$. We
renormalize the row over the real keys, $w_{ij} = a_{ij} / \sum_k
a_{ik}$, so that the mass on the phantom is excluded where one exists,
and use these weights to compute the mean gate value of the row, $m_i =
\sum_j w_{ij}\, g_j$, and the weighted variance $\sum_j w_{ij}\,(g_j -
m_i)^2$. The square root of that variance is the selectivity of the row.
It is zero when every attended read receives the same gate value,
whether open or attenuated, and it grows as the gate passes some of the
attended reads and blocks others. We average the row values over
queries, heads, sequences, and layers, on eight validation batches at
the training context, for the seed-0 model of each variant. 

Gate selectivity provides evidence separating the two primitives at the level of the
gate itself, independently of the validation loss. A gate that supplied only abstention would uniformly 
scale each row as a whole, i.e., $g_j = g$ for all $j$. Thus, the standard deviation of the gate values $g_j$ would be zero within a row (i.e. a single query). 
But filtering requires the opposite: different gate values are applied for different reads within a single query, 
and each value is gated individually. Thus, the standard deviation of the gate values $g_j$ would be positive within a row.
\Cref{fig:trend}(d) shows that gate selectivity emerges between 10M and 50M,
where the benefits of filtering appear, and stays high thereafter as the scale increases.

\subsection{Attention Sinks in the Trained Models}\label{sec:autopsy}

The shrinking abstention increment raises the question of what the
baseline does instead of abstaining, and an analysis of the trained
models answers it. \Cref{fig:autopsy} reports, for each 124M variant,
where its attention mass goes and what happens to the value read at
position 0. The baseline places 6.0\% of its mass on the first key
position and drives the value read there toward zero, with a norm ratio
of 0.42. This is the attention sink \citep{xiao2024streamingllm, gu2025sink,
barbero2025first}. With a sink logit, the mass on
position 0 falls to 0.5\%, the norm ratio recovers to 0.99, and the mass
moves to the phantom, and sink+proj behaves the same way.

\begin{figure}[t]
  \begin{center}
    \centerline{\includegraphics[width=\columnwidth]{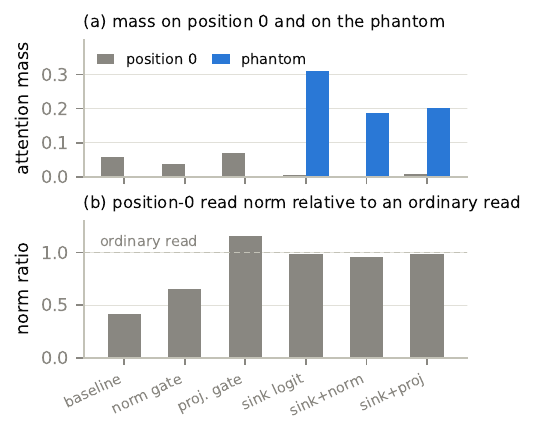}}
    \caption{Where the sink lives at 124M (1024-token context, seed
    0). (a) Attention mass on the first key position, as a fraction of the mass
    on real keys, and on the phantom, as a fraction of all mass. The
    baseline and the norm gate park mass on the first
    key position, every variant with a phantom routes it there
    instead, and the projection gate keeps a position-0 sink of the
    baseline's size. (b) The norm ratio divides the mean norm of the value vector at position
    0, taken before any gate is applied, by the mean norm of the value
    vectors at positions 4 and later in the same head and the same
    sequences, averaged over heads and layers on fixed validation
    blocks. The baseline
    crushes that read, the
    norm gate crushes it less, and every variant with a phantom or a
    projection gate keeps it ordinary.}
    \label{fig:autopsy}
  \end{center}
  \vskip -0.2in
\end{figure}

Both gate-only variants keep a position-0 sink. The norm gate parks
3.7\% of its mass there with a norm ratio of 0.65, and the projection
gate parks 7.1\% with an uncrushed read that the gate closes instead.

The same measurement at 350M
(\cref{tab:sink} in \cref{app:telemetry}) gives the trend with scale.
The baseline parks 11\% of its mass on position 0 with a norm ratio
of 0.36. The sink-logit model parks 0.7\%
there, keeps the norm ratio at 1.01, and places 53\% of its mass on
the phantom. These results suggest that the abstention a large baseline obtains by
imitation comes closer to what a built-in mechanism provides, which would explain why the loss saved by that mechanism shrinks
from 10M to 124M even as the mass it absorbs grows. The production model gpt-oss-20b, which has a sink
logit, shows the same pattern at 20B parameters, with 34\% to 48\% of
its mass on the phantom and an uncrushed position-0 read
(\cref{app:incidence}).

% =====================================================================
\subsection{Additivity of the Two Primitives}\label{sec:additivity}
% Numbers: paper/tables/additivity.tex (tables.py).

This section asks whether abstention and filtering are separate
mechanisms, by measuring how their benefits combine. If a value gate and a sink logit were two descriptions of one mechanism,
supplying both would add nothing over the better one, and if they were
independent their benefits would simply add.
Lemma~\ref{lem:scalar} suggests the reality lies in between, since a
gate used alone also shrinks whole rows through $Z_i$, and a sink
logit can withdraw from some of the same rows.

\begin{table}[t]
  \caption{Two-by-two decomposition of the benefit of each gate form
  (validation-loss improvement over the baseline, nats, paired per
  seed). Sink, gate, and both are the improvements from the sink
  logit alone, the gate alone, and the combined variant. Overlap is sink
  plus gate minus both, filter is both minus sink, and share is filter
  as a percentage of both.}
  \label{tab:additivity}
  \begin{center}
    \begin{footnotesize}
      \setlength{\tabcolsep}{3pt}
      \begin{tabular}{lrrrrrr}
\toprule
Tier & Sink & Gate & Both & Overlap & Filter & Share \\
\midrule
\multicolumn{7}{l}{\emph{Norm gate}} \\
10M & $0.0185$ & $0.0037$ & $0.0177$ & $0.0045$ & $-0.0008$ & $-5\%$ \\
50M & $0.0068$ & $0.0069$ & $0.0106$ & $0.0030$ & $0.0030$ & $28\%$ \\
124M & $0.0052$ & $0.0059$ & $0.0087$ & $0.0024$ & $0.0036$ & $41\%$ \\
350M & $0.0103$ & -- & $0.0187$ & -- & $0.0084$ & $45\%$ \\
\midrule
\multicolumn{7}{l}{\emph{Projection gate}} \\
10M & $0.0185$ & $0.0113$ & $0.0221$ & $0.0076$ & $0.0036$ & $16\%$ \\
50M & $0.0068$ & $0.0094$ & $0.0133$ & $0.0029$ & $0.0065$ & $49\%$ \\
124M & $0.0052$ & $0.0091$ & $0.0106$ & $0.0037$ & $0.0054$ & $51\%$ \\
350M & $0.0103$ & -- & $0.0193$ & -- & $0.0089$ & $46\%$ \\
\bottomrule
\end{tabular}

    \end{footnotesize}
  \end{center}
  \vskip -0.1in
\end{table}

\Cref{tab:additivity} reports the additive combination of the two primitives at each scale. 
It shows the improvements over the baseline from the sink logit alone, from the gate alone, and from both together, as well as the overlap.  
The overlap is the amount by which the sum of the two single-mechanism improvements exceeds the
improvement from both simultaneously. The filter column is the improvement of the
combined variant over the sink logit.

\Cref{tab:additivity} shows that the improvements from the two
primitives are approximately \emph{additive} with a small \emph{overlap}. Five patterns stand out. 
First, the combined variant improves on the baseline by more than
either mechanism alone.
Second, the overlap is small in absolute terms, between 0.0024 and
0.0076 nats at the three tiers where it can be measured.
Third, the filter column is positive at every tier for the projection
gate and from 50M upward for the norm gate, and its share of the combined benefit rises from 16\% at 10M to 51\%
at 124M for the projection gate and from zero to 41\% for the norm
gate. At 350M the shares are 46\% and 45\%. This suggests that filtering grows from a minor
contributor to an equal partner of abstention over this range. 
Fourth, the projection gate's filtering contribution exceeds the
norm gate's at every tier.
Finally, the combined variant with the projection gate has the lowest
mean validation loss at every tier and improves on the sink logit in
every paired seed.

% =====================================================================
\section{The Noise-Injection Experiments for Value Gates}\label{sec:fingerprint}

In this section, we investigate whether the value gates remove noise
from the value vectors, and the degree to which they do, based on a
series of controlled experiments with injected interference in the
value reads at inference time. The idea is to corrupt the reads that a
filter is built to reject, measure the increase in validation loss
under each model variant, and compare. If value gates are filters, a
model with a gate should suffer a smaller increase than a model
without one.

\subsection{Experimental Design}\label{sec:battery}

\paragraph{Corruption.} Every experiment uses the seed-0 checkpoint
of each variant and modifies nothing but the value reads at
evaluation time. For each head and each sequence, we compute the norm
of every value read $v_j$ and set a cutoff at the median or at the
25th percentile of those norms. The reads below the cutoff are the
head's \emph{quiet reads} for that sequence, and they are the target
because they are what a filter is designed to reject and where
interference is most plausible. Each quiet read is then replaced by
$v_j + \epsilon\, m_h\, u_j$, where $m_h$ is the median read norm of
that head in that sequence, $u_j$ is a unit noise vector, and the dose
$\epsilon$ takes the values 0.2, 0.4, 0.8, and 1.6. The corruption is
applied before any gate acts, so every variant sees the same
corrupted reads and differs only in what it does with them.

\paragraph{Junk types.} In the primary specification, which we call
\emph{structured junk}, $u_j$ is the normalized value read of the
same head at a uniformly random other position in the evaluation
batch, which almost always lies in another sequence. Such junk has
the statistical shape of a real read but no relation to the current
context, and it is far more damaging than noise without feature
structure. At 124M with the median cutoff, structured junk at
$\epsilon = 0.4$ costs the baseline 0.48 nats, whereas Gaussian junk
of the same norm costs it 0.032 nats (\cref{app:battery}).
Alternative specifications reported in \cref{app:battery} draw $u_j$
from the same sequence, from an isotropic Gaussian, from a random
fixed direction per head, and from the gate's own learned direction.

\paragraph{Evaluation.} Every evaluation uses the same 64 validation
sequences, drawn once with a fixed seed as eight batches of eight random
windows of the validation split, each window as long as the model's
training context, 512 tokens at 10M and 1{,}024 tokens at the other
tiers, for a total of 65{,}536 evaluated tokens at the larger tiers. The
loss is the mean next-token cross-entropy over every position of those
sequences. The outcome measure is that loss under corruption minus the
loss of the same model on the same sequences without corruption, so that
each entry isolates the effect of the injected junk.

\subsection{Results: Two Routes to Robustness}\label{sec:fingerprintcurves}

\Cref{tab:injectionmain} reports two kinds of measurement for every
variant at 124M and 350M. The four dose columns give the increase in
loss under structured junk at the 25th-percentile cutoff. The two mass
columns describe, before any corruption, where each variant's
attention goes. For every query, we take the attention weights over
the keys, including the phantom where one exists, so that they sum
to one, and record the total weight on the keys whose read norm falls
below the head's 25th percentile in that sequence, which is the share
of the query's attention exposed to the injection, and the weight on
the phantom. Both are averaged over all queries, heads, and layers on
the same eight validation batches, at seed 0 and the training context.
Three readings of the table follow.

\begin{table}[t]
  \caption{Injected interference at 124M and 350M, seed 0. Quiet is
  the share of a query's attention mass, phantom included, that lands
  on reads below the head's 25th-percentile norm in that sequence,
  which is the share exposed to the injection, and phantom is the share on
  the phantom, zero by construction for variants without one, each
  averaged over all queries, heads, and layers on eight validation
  batches at the training context. The
  dose columns give the mean increase in validation loss (nats) on the
  same batches under structured junk injected into those reads at dose $\epsilon$. The
  gate-only variants were not trained at 350M.}
  \label{tab:injectionmain}
  \begin{center}
    \begin{footnotesize}
      \setlength{\tabcolsep}{3pt}
      \begin{tabular}{lrrrrrr}
\toprule
 & \multicolumn{2}{c}{mass (\%)} & \multicolumn{4}{c}{loss increase at dose $\epsilon$} \\
\cmidrule(lr){2-3} \cmidrule(lr){4-7}
Variant & quiet & phantom & 0.2 & 0.4 & 0.8 & 1.6 \\
\midrule
\multicolumn{7}{l}{\emph{124M}} \\
baseline & $36$ & $0$ & $0.018$ & $0.110$ & $2.63$ & $4.95$ \\
sink logit & $20$ & $31$ & $0.005$ & $0.021$ & $0.13$ & $3.16$ \\
norm gate & $44$ & $0$ & $0.004$ & $0.042$ & $4.40$ & $6.82$ \\
projection gate & $22$ & $0$ & $0.006$ & $0.029$ & $0.26$ & $1.89$ \\
sink+norm & $30$ & $19$ & $0.004$ & $0.020$ & $0.51$ & $5.54$ \\
sink+proj & $19$ & $20$ & $0.004$ & $0.020$ & $0.18$ & $1.77$ \\
\midrule
\multicolumn{7}{l}{\emph{350M}} \\
baseline & $46$ & $0$ & $0.019$ & $0.119$ & $0.86$ & $5.92$ \\
sink logit & $13$ & $53$ & $0.004$ & $0.015$ & $0.10$ & $5.05$ \\
sink+norm & $27$ & $35$ & $0.002$ & $0.014$ & $0.65$ & $6.48$ \\
sink+proj & $17$ & $23$ & $0.003$ & $0.012$ & $0.05$ & $1.10$ \\
\bottomrule
\end{tabular}

    \end{footnotesize}
  \end{center}
  \vskip -0.1in
\end{table}

First, every variant with a sink logit or a gate loses a fraction of
the baseline's damage at low and moderate doses. At 124M and
$\epsilon = 0.4$, the baseline loses 0.110 nats, the sink logit
0.021, the norm gate 0.042, the projection gate 0.029, and the two
combined variants 0.020. The pattern is the same at 350M and at 50M
(\cref{app:battery}).

Second, a model can incur less loss in two ways, by reading less of the
corrupted material or by cleaning what it reads, and
\cref{tab:injectionmain} separates them. The sink logit supplies no filter. Its phantom holds
31\% of its attention mass, so its share on the corrupted reads falls
to 20\% from the baseline's 36\%, and that is the whole source of its
robustness. At 350M the same holds, with 13\% attention mass on quiet reads for the
sink logit against 46\% for the baseline. The norm gate is the mirror case. It places 44\% of its mass on the corrupted reads, more than the
baseline, and still incurs only a quarter to a half as much loss as the
baseline at low doses. This is indicative of the gate itself performing the filtering function. The projection gate and the combined
variants do some of each. The sink+proj variant is the most robust at
the highest dose at both tiers and within 0.05 nats of the best
variant at every other dose.

Third, the norm gate's row reverses at $\epsilon = 0.8$, where it
loses 4.40 nats against the baseline's 2.63, because a threshold on
norm passes any junk large enough to lift a read above it. The
sink+norm variant inherits this reversal at 350M at the highest dose.
The projection gate has the corresponding failure only for junk
aligned with its own gate direction, which a fixed-direction
injection in \cref{app:battery} shows.

In summary, the injection experiments support three conclusions.
Every variant with a sink logit or a gate incurs only a fraction of the
baseline's loss increase at low and moderate doses. That protection
arrives by two routes that the exposure measurements separate,
abstention, which reads less of the corrupted material, and filtering, which removes it from the value projection
pathway. The two routes combine to deliver the greatest reduction in
losses, so that sink+proj is the most robust variant at the
highest dose at both tiers, whereas the norm gate's reversal at high
doses marks the limit of a filter determined by norm alone.

% =====================================================================
\section{Discussion and Conclusion}\label{sec:discussion}

\paragraph{Relation to prior explanations.} Within the range we train,
the benefit a gate delivers shifts from almost entirely abstention at
the smallest tier to an equal share of filtering at the largest, so
studies that
measure gates at different scales can reasonably attribute the gain to
different causes. In these terms, added nonlinearity is a \emph{noise filter}. The other three proposed causes all concern \emph{abstention}: the elimination of attention sinks, the
mitigation of the extreme-token phenomena that accompany sinks
\citep{bu2025vga}, and the ability of a head to perform no operation. The query-computed output gate of \citet{qiu2025gated} can abstain per query, consistent with their observation that it removes sinks. The finding of \citet{bu2025vga} that a value gate outperforms a
learnable sink at 125M is consistent with the positive filtering
increment of the projection gate at 124M (\cref{tab:additivity}). 

\paragraph{Limitations.} First, due to limited computing resources, our loss measurements stop at
350M parameters, on one corpus, FineWeb-Edu, at modest token budgets.
This range brackets the scale at which value-state gated attention was
evaluated \citep{bu2025vga}. The filtering increment grows
across our tiers and matches the abstention increment at 350M, but
whether it
keeps growing, levels off, or gives way to something else at billions of parameters is unknown. The pretrained models of up to 20B
parameters in \cref{app:incidence} confirm the sink behaviour and the
injection damage, not the size of any loss increment. Second, the filtering
increment includes the gate's row scaling $Z_i$, counted as filtering
because it is determined by the reads, and a loss measurement cannot
separate interference removal from other uses of that scale.

\paragraph{Conclusion.} Softmax attention lacks two primitives,
abstention and noise filtering, and a learned gate on the value pathway
supplies part of each. Giving each primitive its own mechanism in
matched models of 10M to 350M parameters shows that the composition of
the gate's benefit changes with scale, from almost entirely abstention
at 10M to an equal share of filtering at 350M, and that the two
benefits are
largely additive with a small overlap.
Injected interference separates the two routes by which they protect a
model, abstention by reading less of the corrupted material and
filtering by removing it from the value projection pathway. A head with a sink logit for abstention \citep{openai2025gptoss} and a
projection gate for filtering \citep{bu2025vga} obtains both benefits, with the lowest mean loss at every tier we
train and the greatest robustness at high doses of injected
interference, at a cost of $H(d+2)$ parameters per layer and no change
to the key-value cache.

% ---------------------------------------------------------------------
\section*{Impact Statement}
This paper presents work whose goal is to advance the field of Machine
Learning. There are many potential societal consequences of our work,
none which we feel must be specifically highlighted here.

\bibliography{references}
\bibliographystyle{icml2026}

% =====================================================================
\newpage
\appendix
\crefalias{section}{appendix}
\onecolumn

\section{Training Details}\label{app:training}

\paragraph{Architecture.}
All models are decoder-only transformers in the GPT-2 style
\citep{radford2019gpt2}. Each block applies layer normalization before
attention and before the
feed-forward network (pre-norm). The feed-forward network has a hidden
width of four times the model width with a GELU nonlinearity, and no
linear layer carries a bias.
Positional information is a learned embedding, the token embedding
is tied to the output projection, and dropout is not used. Weights
are initialized from a normal distribution with standard deviation
0.02, except that the output projections of attention and of the
feed-forward network use a standard deviation of $0.02/\sqrt{2L}$
for $L$ layers. The vocabulary is the GPT-2 byte-pair encoding with
50{,}257 tokens.

\paragraph{Data.}
The four tiers are trained on FineWeb-Edu
\citep{penedo2024fineweb}, taken from the \texttt{sample-10BT}
subset, tokenized with the GPT-2 tokenizer, and concatenated with an
end-of-text token between documents. A validation set of 5.2M tokens
is taken first, so that it is disjoint from every training set. The
training sets contain 950M tokens for the 10M and 50M tiers, 1.6B
tokens for the 124M tier, and 6.0B tokens for the 350M tier. Training
batches are random windows of the context length drawn with
replacement from the training stream, so no tier sees its training
set more than once in expectation.

\paragraph{Optimization.}
Every tier uses AdamW with $\beta_1 = 0.9$, $\beta_2 = 0.95$, a peak
learning rate of $6 \times 10^{-4}$, a linear warmup, and a cosine
decay to $6 \times 10^{-5}$ at the final iteration. Weight decay of
0.1 is applied to every parameter tensor of dimension two or more,
which includes the embedding, and to no others. In particular, the
norm-gate threshold $\tau_h$, the projection-gate
direction $w_h$ and bias $b_h$, and the sink logit $b_h$ are excluded
from weight decay. Decay would pull each of them toward zero, and zero has a mechanistic
meaning in each case, no abstention for the sink logit and uniform
attenuation for the gates. Gradients are clipped to a
global norm of 1.
Training runs under bfloat16 autocast on the matrix multiplications,
and the gate computations, namely the read norms and the sigmoids, are performed in single
precision. \Cref{tab:hyper} gives the per-tier settings. The
validation loss is estimated on a fixed set of batches drawn with a
fixed seed, sixteen batches during training and sixty-four for the
final value, so that every variant and every seed is evaluated on the
same tokens.

\begin{table}[h]
  \caption{Per-tier training configuration. Batch is the number of
  sequences per optimizer step, written as micro-batch size times
  accumulation steps.}
  \label{tab:hyper}
  \begin{center}
    \begin{small}
      \begin{tabular}{lrrrrrrrr}
        \toprule
        Tier & Layers & Heads & Width & Context & Batch & Iterations &
        Warmup & Tokens \\
        \midrule
        10M  & 6  & 6  & 384  & 512  & $64 \times 1$ & 6{,}000   & 200
        & 0.20B \\
        50M  & 10 & 10 & 640  & 1024 & $16 \times 2$ & 29{,}000  & 500
        & 0.95B \\
        124M & 12 & 12 & 768  & 1024 & $8 \times 4$  & 48{,}800  & 800
        & 1.60B \\
        350M & 24 & 16 & 1024 & 1024 & $8 \times 4$  & 177{,}000 &
        2{,}000 & 5.80B \\
        \bottomrule
      \end{tabular}
    \end{small}
  \end{center}
\end{table}

\paragraph{Seeds and pairing.}
The seed sets the parameter initialization and the order of training
batches, and it is shared across variants. Two variants with the same
seed therefore see the same tokens in the same order and differ only in
their attention mechanism. This is what makes the paired per-seed
differences in the main text meaningful. The three sink+proj seeds at
10M were trained after the other 10M variants with a micro-batch of 16
and four accumulation steps rather than a single batch of 64. Because
the batch sampler draws indices from the same generator in the
same order, this produces the identical sequence of training batches and
the identical evaluation batches. We verified that it does.

\paragraph{Implementation of the variants.}
The phantom key of the off-by-one, sink-logit, and combined variants is
implemented as an extra key and value column at position zero, with the value fixed at zero and the column visible to every query. The
column is kept visible either by an explicit attention mask or,
equivalently, by prepending a dummy query row and using the fused causal
kernel, and the two compute the same attention. For the off-by-one variant the phantom key is the zero
vector, so its logit is zero. For the variants with a learnable sink
logit, the head dimension of the queries and keys is augmented. The
query receives a constant coordinate equal to $\sqrt{d}$, the
phantom key receives the sink logit $b_h$ in that coordinate, and every
real key receives zero. With the attention scale set to $1/\sqrt{d}$
explicitly, the phantom's logit is $b_h$ and every other logit
is unchanged. The augmented dimension is zero-padded to a multiple of eight so that
fused attention kernels accept it. The sink-token variant adds a
learned extra token at the start of every input sequence and removes
the corresponding output position before the language-model head.

\paragraph{The variants.}
\Cref{tab:variants} summarizes the eight variants of \cref{sec:armlist}
with the parameter cost of each.

\begin{table}[h]
  \caption{The eight variants. Parameters are counted per layer for $H$
  heads of dimension $d$. Every variant is compatible with the key-value
  cache, because each gate is a function of the read it multiplies.
  The gates have no routing-side abstention. Their row scaling $Z_i$
  (Lemma~\ref{lem:scalar}) is noted because it can shrink a row's output.}
  \label{tab:variants}
  \begin{center}
    \begin{small}
      \begin{tabular}{lllr}
        \toprule
        Variant & Abstention & Filter & Params \\
        \midrule
        baseline   & none              & none                  & 0 \\
        offbyone   & phantom, fixed    & none                  & 0 \\
        sinklogit  & phantom, learned  & none                  & $H$ \\
        sinktoken  & token, approximate& none                  &
        $d_{\text{model}}$ \\
        normgate   & none; $Z_i$ scaling & norm threshold        & $H$ \\
        projgate   & none; $Z_i$ scaling & learned direction     & $H(d+1)$ \\
        sink+norm  & phantom, learned  & norm threshold        & $2H$ \\
        sink+proj  & phantom, learned  & learned direction     & $H(d+2)$ \\
        \bottomrule
      \end{tabular}
    \end{small}
  \end{center}
\end{table}

\paragraph{Compute.}
All training was carried out on a single NVIDIA GeForce RTX 5090
with PyTorch 2.11. A 10M run takes about ten minutes, and a 350M run
takes between 29 and 47 hours depending on the variant and on the
attention kernel used. The total compute of the study is dominated by
the 350M runs and the three seeds at each of the 50M and 124M
tiers.

\section{Full Variant Tables}\label{app:tables}

\begin{table*}[t]
  \caption{Final validation loss (nats) across four model sizes. Each
  improvement is the mean of per-seed paired differences with its
  standard error, and negative is better. The last block measures what each gate adds once a sink logit is in
  place. Variants absent at a tier are marked with a dash.}
  \label{tab:loss}
  \begin{center}
    \begin{small}
      \begin{tabular}{lrrrr}
\toprule
Variant & 10M & 50M & 124M & 350M \\
\midrule
\multicolumn{5}{l}{\emph{Baseline loss}} \\
baseline & $4.2119$ & $3.5399$ & $3.3710$ & $3.0293$ \\
\midrule
\multicolumn{5}{l}{\emph{Improvement over baseline (paired per seed)}} \\
off-by-one & $-0.0160 \pm 0.0027$ & -- & -- & -- \\
sink token & $-0.0025 \pm 0.0041$ & -- & -- & -- \\
sink logit & $-0.0185 \pm 0.0033$ & $-0.0068 \pm 0.0003$ & $-0.0052 \pm 0.0025$ & $-0.0103 \pm 0.0084$ \\
norm gate & $-0.0037 \pm 0.0026$ & $-0.0069 \pm 0.0008$ & $-0.0059 \pm 0.0012$ & -- \\
projection gate & $-0.0113 \pm 0.0029$ & $-0.0094 \pm 0.0011$ & $-0.0091 \pm 0.0011$ & -- \\
sink+norm & $-0.0177 \pm 0.0037$ & $-0.0106 \pm 0.0015$ & $-0.0087 \pm 0.0021$ & $-0.0187 \pm 0.0060$ \\
sink+proj & $-0.0221 \pm 0.0029$ & $-0.0133 \pm 0.0018$ & $-0.0106 \pm 0.0022$ & $-0.0193 \pm 0.0057$ \\
\midrule
\multicolumn{5}{l}{\emph{Improvement over sink logit (filtering with abstention held fixed)}} \\
norm gate & $+0.0148 \pm 0.0016$ & $-0.0001 \pm 0.0007$ & $-0.0008 \pm 0.0018$ & -- \\
projection gate & $+0.0072 \pm 0.0053$ & $-0.0026 \pm 0.0010$ & $-0.0039 \pm 0.0030$ & -- \\
sink+norm & $+0.0008 \pm 0.0013$ & $-0.0030 \pm 0.0009$ & $-0.0036 \pm 0.0006$ & $-0.0084 \pm 0.0039$ \\
sink+proj & $-0.0036 \pm 0.0008$ & $-0.0065 \pm 0.0016$ & $-0.0054 \pm 0.0012$ & $-0.0089 \pm 0.0027$ \\
\bottomrule
\end{tabular}

    \end{small}
  \end{center}
  \vskip -0.1in
\end{table*}

\Cref{tab:loss} reports the paired improvements at every tier, and
\cref{tab:allseeds} lists the final validation loss of every run in
the study. The sink token was run only at 10M, where its improvement over the
baseline is within seed noise (\cref{tab:loss}).

\begin{table}[h]
  \caption{Final validation loss (nats) of every run. Dashes mark
  seeds that were not trained.}
  \label{tab:allseeds}
  \begin{center}
    \begin{small}
      \begin{tabular}{llrrrrrr}
\toprule
Tier & Variant & seed 0 & seed 1 & seed 2 & seed 3 & seed 4 & mean \\
\midrule
FineWeb 10M & baseline & $4.2141$ & $4.2060$ & $4.2155$ & -- & -- & $4.2119$ \\
 & off-by-one & $4.1988$ & $4.1944$ & $4.1945$ & -- & -- & $4.1959$ \\
 & sink token & $4.2037$ & $4.2061$ & $4.2184$ & -- & -- & $4.2094$ \\
 & sink logit & $4.1981$ & $4.1916$ & $4.1904$ & -- & -- & $4.1934$ \\
 & norm gate & $4.2099$ & $4.2071$ & $4.2076$ & -- & -- & $4.2082$ \\
 & projection gate & $4.2070$ & $4.1890$ & $4.2057$ & -- & -- & $4.2006$ \\
 & sink+norm & $4.2014$ & $4.1906$ & $4.1905$ & -- & -- & $4.1942$ \\
 & sink+proj & $4.1952$ & $4.1864$ & $4.1876$ & -- & -- & $4.1897$ \\
\midrule
50M & baseline & $3.5417$ & $3.5399$ & $3.5381$ & -- & -- & $3.5399$ \\
 & sink logit & $3.5347$ & $3.5338$ & $3.5309$ & $3.5285$ & $3.5279$ & $3.5311$ \\
 & norm gate & $3.5359$ & $3.5337$ & $3.5295$ & -- & -- & $3.5330$ \\
 & projection gate & $3.5301$ & $3.5319$ & $3.5295$ & -- & -- & $3.5305$ \\
 & sink+norm & $3.5282$ & $3.5314$ & $3.5282$ & $3.5264$ & $3.5264$ & $3.5281$ \\
 & sink+proj & $3.5249$ & $3.5287$ & $3.5261$ & -- & -- & $3.5266$ \\
\midrule
124M & baseline & $3.3747$ & $3.3681$ & $3.3700$ & -- & -- & $3.3710$ \\
 & sink logit & $3.3646$ & $3.3663$ & $3.3664$ & -- & -- & $3.3658$ \\
 & norm gate & $3.3666$ & $3.3623$ & $3.3661$ & -- & -- & $3.3650$ \\
 & projection gate & $3.3658$ & $3.3571$ & $3.3627$ & -- & -- & $3.3619$ \\
 & sink+norm & $3.3618$ & $3.3615$ & $3.3633$ & -- & -- & $3.3622$ \\
 & sink+proj & $3.3600$ & $3.3585$ & $3.3626$ & -- & -- & $3.3604$ \\
\midrule
350M & baseline & $3.0338$ & $3.0284$ & $3.0257$ & -- & -- & $3.0293$ \\
 & sink logit & $3.0325$ & $3.0013$ & $3.0230$ & -- & -- & $3.0190$ \\
 & sink+norm & $3.0251$ & $2.9991$ & $3.0075$ & -- & -- & $3.0106$ \\
 & sink+proj & $3.0210$ & $2.9979$ & $3.0112$ & -- & -- & $3.0100$ \\
\bottomrule
\end{tabular}

    \end{small}
  \end{center}
\end{table}

\paragraph{Seed variance.} Two observations from \cref{tab:allseeds}
bear on the reading of the main-text tables. At 350M the three seeds of the sink logit give paired improvements
of 0.0013, 0.0271, and 0.0026 nats over the baseline, so its 350M
entries in \cref{tab:loss} and \cref{tab:lossmain} carry a standard
error close to their mean, whereas the two combined variants improve
on the sink logit in every seed. The 10M
projection-gate variant has a
seed standard deviation of 0.010, three times that of any other variant
at that tier, so its 10M entries in \cref{tab:loss} carry that
variance.

\paragraph{Routing-slot control.}
Lemma~\ref{lem:scalar} separates a value gate into a gate in the routing
slot and the row scaling $Z_i$. To measure whether the redistribution
of weight among the reads carries any benefit on its own, we trained
the norm gate in the routing slot, with routing weights
$\mathrm{softmax}_j(s_{ij} + \log g_j)$ and no factor $Z_i$, with
three seeds at 10M, 50M, and 124M. The term $\log g_j$ is folded into
one extra key dimension that is paired with a constant query entry, so
that the fused attention kernel computes the modified logits.
\Cref{tab:routing} reports the paired improvement over the baseline,
the mean gate value over the reads of the last validation batch at
the end of training, averaged over layers, and the fraction of those
reads with a gate value below one half. The gate opens
almost fully at every tier and the improvement is within seed noise.
The redistribution therefore carries no measurable benefit on its
own, and the benefit of the same gate in the value slot, 0.0037 to
0.0069 nats at these tiers (\cref{tab:lossmain}), arrives together
with the row scaling $Z_i$. A gate that must keep the routing mass at
one cannot remove a read, only move its weight to other reads. These
results suggest that removal is what a value gate is used for.

\begin{table}[h]
  \caption{The norm gate in the routing slot. Improvement is the mean
  of per-seed paired differences from the baseline with its standard
  error, and positive is better. The gate value and the fraction of
  reads below one half are averaged over the three seeds.}
  \label{tab:routing}
  \begin{center}
    \begin{small}
      \begin{tabular}{lrrr}
\toprule
Tier & improvement over baseline & mean gate value & reads below $0.5$ \\
\midrule
10M & $-0.0017 \pm 0.0014$ & $0.999$ & $0.0\%$ \\
50M & $+0.0017 \pm 0.0002$ & $0.999$ & $0.0\%$ \\
124M & $-0.0001 \pm 0.0005$ & $0.987$ & $0.7\%$ \\
\bottomrule
\end{tabular}

    \end{small}
  \end{center}
\end{table}

\section{Injection Experiment Details}\label{app:battery}

\paragraph{Procedure.}
The corruption, the dose scaling, and the evaluation set are as
described in \cref{sec:battery}, and every experiment uses the seed-0
checkpoint of each variant.

\paragraph{Junk types.}
In the structured condition, $u_j$ is a normalized value read taken from
a uniformly random other position in the batch, which almost always lies
in another sequence, so that the junk has the statistical shape of a real read but no
relation to the current context. In the same-context condition the
read is taken from a random position of the same sequence. In the
Gaussian condition $u_j$ is an isotropic normal vector normalized to
unit length. In the random-direction condition $u_j$ is a single
unit vector per head drawn at random and shared by every corrupted
read in the batch. In the own-direction condition $u_j$ is the
head's learned gate direction $w_h / \norm{w_h}$, which is defined
only for the projection-gate variants. \Cref{tab:batterya,tab:batteryb}
list every condition at every tier.

\begin{figure*}[t]
  \begin{center}
    \centerline{\includegraphics[width=\textwidth]{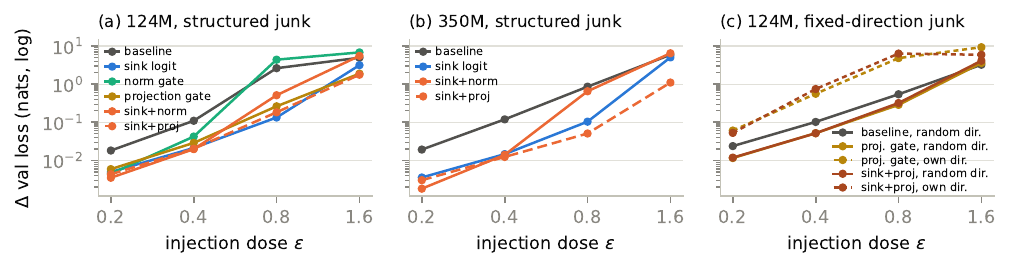}}
    \caption{Robustness to injected junk.
    (a,b) Increase in validation loss (log scale) under structured
    junk injected into the quietest quarter of each head's reads, at
    124M and 350M. Every variant with a sink logit or a gate is several
    times more robust than the baseline at low and moderate doses. At
    high doses the norm gate, and the sink+norm variant that contains it,
    become worse than the baseline, whereas the projection gate does
    not. (c) At 124M, junk along a random fixed direction hurts the
    projection gate no more than the baseline, but junk along the
    gate's own learned direction collapses it, and collapses sink+proj
    further.}
    \label{fig:battery}
  \end{center}
\end{figure*}

\begin{table}[h]
  \caption{Full injection battery at the 10M and 50M tiers. Entries
  are the increase in validation loss (nats) over the uncorrupted
  model.}
  \label{tab:batterya}
  \begin{center}
    \begin{footnotesize}
      \begin{tabular}{lllrrrr}
\toprule
Tier & Junk type (cutoff) & Variant & $\epsilon=0.2$ & $\epsilon=0.4$ & $\epsilon=0.8$ & $\epsilon=1.6$ \\
\midrule
FineWeb 10M & structured (other context) (q25) & baseline & $+0.032$ & $+0.138$ & $+0.771$ & $+2.695$ \\
 &  & sink logit & $+0.009$ & $+0.041$ & $+0.178$ & $+0.778$ \\
 &  & norm gate & $+0.027$ & $+0.124$ & $+0.877$ & $+3.162$ \\
 &  & projection gate & $+0.026$ & $+0.100$ & $+0.408$ & $+1.368$ \\
 &  & sink+norm & $+0.010$ & $+0.047$ & $+0.252$ & $+1.685$ \\
 &  & sink+proj & $+0.007$ & $+0.030$ & $+0.133$ & $+0.594$ \\
\midrule
FineWeb 10M & structured (other context) (median) & baseline & $+0.069$ & $+0.310$ & $+1.948$ & $+3.239$ \\
 &  & sink logit & $+0.019$ & $+0.093$ & $+0.474$ & $+2.485$ \\
 &  & norm gate & $+0.063$ & $+0.310$ & $+2.122$ & $+3.681$ \\
 &  & projection gate & $+0.060$ & $+0.258$ & $+1.013$ & $+2.962$ \\
 &  & sink+norm & $+0.023$ & $+0.117$ & $+0.695$ & $+3.104$ \\
 &  & sink+proj & $+0.017$ & $+0.077$ & $+0.373$ & $+1.845$ \\
\midrule
FineWeb 10M & Gaussian (median) & baseline & $+0.009$ & $+0.034$ & $+0.142$ & $+0.589$ \\
 &  & sink logit & $+0.004$ & $+0.016$ & $+0.062$ & $+0.264$ \\
 &  & norm gate & $+0.005$ & $+0.026$ & $+0.196$ & $+0.940$ \\
 &  & projection gate & $+0.003$ & $+0.015$ & $+0.074$ & $+0.346$ \\
 &  & sink+norm & $+0.004$ & $+0.017$ & $+0.077$ & $+0.429$ \\
 &  & sink+proj & $+0.003$ & $+0.012$ & $+0.051$ & $+0.224$ \\
\midrule
50M & structured (other context) (q25) & baseline & $+0.018$ & $+0.103$ & $+1.296$ & $+4.688$ \\
 &  & sink logit & $+0.007$ & $+0.032$ & $+0.182$ & $+3.124$ \\
 &  & norm gate & $+0.008$ & $+0.070$ & $+3.242$ & $+5.225$ \\
 &  & projection gate & $+0.007$ & $+0.033$ & $+0.199$ & $+1.684$ \\
 &  & sink+norm & $+0.004$ & $+0.024$ & $+0.371$ & $+5.997$ \\
 &  & sink+proj & $+0.005$ & $+0.021$ & $+0.096$ & $+0.944$ \\
\midrule
50M & structured (other context) (median) & baseline & $+0.034$ & $+0.261$ & $+3.389$ & $+4.920$ \\
 &  & sink logit & $+0.014$ & $+0.070$ & $+1.435$ & $+4.110$ \\
 &  & norm gate & $+0.024$ & $+0.259$ & $+5.358$ & $+5.458$ \\
 &  & projection gate & $+0.016$ & $+0.090$ & $+0.968$ & $+4.396$ \\
 &  & sink+norm & $+0.011$ & $+0.068$ & $+2.564$ & $+6.411$ \\
 &  & sink+proj & $+0.012$ & $+0.056$ & $+0.461$ & $+3.748$ \\
\midrule
50M & Gaussian (median) & baseline & $+0.009$ & $+0.033$ & $+0.137$ & $+0.626$ \\
 &  & sink logit & $+0.004$ & $+0.013$ & $+0.054$ & $+0.259$ \\
 &  & norm gate & $+0.002$ & $+0.015$ & $+0.312$ & $+4.345$ \\
 &  & projection gate & $+0.004$ & $+0.013$ & $+0.052$ & $+0.272$ \\
 &  & sink+norm & $+0.003$ & $+0.013$ & $+0.120$ & $+2.861$ \\
 &  & sink+proj & $+0.002$ & $+0.011$ & $+0.046$ & $+0.241$ \\
\midrule
50M & random direction (median) & baseline & $+0.024$ & $+0.105$ & $+0.542$ & $+2.842$ \\
 &  & sink logit & $+0.016$ & $+0.066$ & $+0.350$ & $+2.734$ \\
 &  & norm gate & $+0.009$ & $+0.056$ & $+1.216$ & $+5.568$ \\
 &  & projection gate & $+0.013$ & $+0.055$ & $+0.285$ & $+2.368$ \\
 &  & sink+norm & $+0.010$ & $+0.053$ & $+0.843$ & $+5.962$ \\
 &  & sink+proj & $+0.013$ & $+0.055$ & $+0.297$ & $+2.838$ \\
\midrule
50M & structured (same context) (median) & baseline & $+0.031$ & $+0.257$ & $+3.330$ & $+4.901$ \\
 &  & sink logit & $+0.012$ & $+0.071$ & $+1.692$ & $+4.281$ \\
 &  & norm gate & $+0.024$ & $+0.262$ & $+5.316$ & $+5.477$ \\
 &  & projection gate & $+0.014$ & $+0.089$ & $+0.973$ & $+4.434$ \\
 &  & sink+norm & $+0.009$ & $+0.069$ & $+2.845$ & $+6.482$ \\
 &  & sink+proj & $+0.007$ & $+0.051$ & $+0.495$ & $+3.733$ \\
\midrule
50M & own gate direction (median) & projection gate & $+0.059$ & $+0.540$ & $+3.611$ & $+11.709$ \\
 &  & sink+proj & $+0.112$ & $+1.041$ & $+11.469$ & $+13.258$ \\
\bottomrule
\end{tabular}

    \end{footnotesize}
  \end{center}
\end{table}

\begin{table}[h]
  \caption{Full injection battery at the 124M and 350M tiers.}
  \label{tab:batteryb}
  \begin{center}
    \begin{footnotesize}
      \begin{tabular}{lllrrrr}
\toprule
Tier & Junk type (cutoff) & Variant & $\epsilon=0.2$ & $\epsilon=0.4$ & $\epsilon=0.8$ & $\epsilon=1.6$ \\
\midrule
124M & structured (other context) (q25) & baseline & $+0.018$ & $+0.110$ & $+2.632$ & $+4.948$ \\
 &  & sink logit & $+0.005$ & $+0.021$ & $+0.134$ & $+3.155$ \\
 &  & norm gate & $+0.004$ & $+0.042$ & $+4.403$ & $+6.824$ \\
 &  & projection gate & $+0.006$ & $+0.029$ & $+0.263$ & $+1.891$ \\
 &  & sink+norm & $+0.004$ & $+0.020$ & $+0.513$ & $+5.544$ \\
 &  & sink+proj & $+0.004$ & $+0.020$ & $+0.183$ & $+1.774$ \\
\midrule
124M & structured (other context) (median) & baseline & $+0.035$ & $+0.479$ & $+4.163$ & $+5.417$ \\
 &  & sink logit & $+0.010$ & $+0.056$ & $+1.545$ & $+5.107$ \\
 &  & norm gate & $+0.016$ & $+0.238$ & $+5.114$ & $+7.328$ \\
 &  & projection gate & $+0.012$ & $+0.096$ & $+1.333$ & $+3.825$ \\
 &  & sink+norm & $+0.009$ & $+0.069$ & $+4.773$ & $+5.835$ \\
 &  & sink+proj & $+0.010$ & $+0.046$ & $+0.843$ & $+5.294$ \\
\midrule
124M & Gaussian (median) & baseline & $+0.008$ & $+0.032$ & $+0.140$ & $+0.692$ \\
 &  & sink logit & $+0.004$ & $+0.013$ & $+0.053$ & $+0.256$ \\
 &  & norm gate & $+0.002$ & $+0.014$ & $+0.297$ & $+4.840$ \\
 &  & projection gate & $+0.003$ & $+0.012$ & $+0.049$ & $+0.268$ \\
 &  & sink+norm & $+0.002$ & $+0.011$ & $+0.161$ & $+3.900$ \\
 &  & sink+proj & $+0.002$ & $+0.010$ & $+0.045$ & $+0.249$ \\
\midrule
124M & random direction (median) & baseline & $+0.024$ & $+0.102$ & $+0.544$ & $+3.240$ \\
 &  & sink logit & $+0.013$ & $+0.062$ & $+0.367$ & $+3.463$ \\
 &  & norm gate & $+0.007$ & $+0.050$ & $+1.433$ & $+6.211$ \\
 &  & projection gate & $+0.012$ & $+0.052$ & $+0.287$ & $+3.608$ \\
 &  & sink+norm & $+0.007$ & $+0.044$ & $+0.910$ & $+5.773$ \\
 &  & sink+proj & $+0.012$ & $+0.052$ & $+0.320$ & $+4.127$ \\
\midrule
124M & structured (same context) (median) & baseline & $+0.033$ & $+0.482$ & $+4.170$ & $+5.337$ \\
 &  & sink logit & $+0.009$ & $+0.054$ & $+1.543$ & $+5.135$ \\
 &  & norm gate & $+0.014$ & $+0.230$ & $+5.165$ & $+7.354$ \\
 &  & projection gate & $+0.012$ & $+0.087$ & $+1.319$ & $+3.766$ \\
 &  & sink+norm & $+0.008$ & $+0.074$ & $+4.568$ & $+6.006$ \\
 &  & sink+proj & $+0.008$ & $+0.047$ & $+0.924$ & $+5.191$ \\
\midrule
124M & own gate direction (median) & projection gate & $+0.061$ & $+0.562$ & $+4.831$ & $+9.331$ \\
 &  & sink+proj & $+0.053$ & $+0.746$ & $+6.361$ & $+5.919$ \\
\midrule
350M & structured (other context) (q25) & baseline & $+0.019$ & $+0.119$ & $+0.861$ & $+5.916$ \\
 &  & sink logit & $+0.004$ & $+0.015$ & $+0.104$ & $+5.047$ \\
 &  & sink+norm & $+0.002$ & $+0.014$ & $+0.653$ & $+6.478$ \\
 &  & sink+proj & $+0.003$ & $+0.012$ & $+0.051$ & $+1.104$ \\
\bottomrule
\end{tabular}

    \end{footnotesize}
  \end{center}
\end{table}

\paragraph{Exposure and downstream tables.}
\Cref{tab:exposure} gives the exposure measurements discussed in
\cref{sec:fingerprintcurves}, and \cref{tab:downstream} the zero-shot
evaluation summarized in \cref{sec:setup4}.

\begin{table}[h]
  \caption{Exposure to injected junk. Share of each attention row's total mass
  placed on quiet reads (below the per-head median or 25th percentile
  norm) and on the phantom, averaged over heads and rows at seed 0 and
  measured at the training context, as in the injection experiments. Mass on the phantom is mass not exposed
  to corruption.}
  \label{tab:exposure}
  \begin{center}
    \begin{small}
      \begin{tabular}{lrrrrrrrrr}
\toprule
 & \multicolumn{3}{c}{50M} & \multicolumn{3}{c}{124M} & \multicolumn{3}{c}{350M} \\
\cmidrule(lr){2-4} \cmidrule(lr){5-7} \cmidrule(lr){8-10}
Variant & median & q25 & phantom & median & q25 & phantom & median & q25 & phantom \\
\midrule
baseline & $0.551$ & $0.349$ & -- & $0.563$ & $0.365$ & -- & $0.632$ & $0.458$ & -- \\
sink logit & $0.361$ & $0.215$ & $0.273$ & $0.337$ & $0.197$ & $0.311$ & $0.224$ & $0.132$ & $0.530$ \\
norm gate & $0.604$ & $0.418$ & -- & $0.619$ & $0.443$ & -- & -- & -- & -- \\
projection gate & $0.431$ & $0.232$ & -- & $0.419$ & $0.223$ & -- & -- & -- & -- \\
sink+norm & $0.458$ & $0.299$ & $0.156$ & $0.451$ & $0.298$ & $0.186$ & $0.384$ & $0.267$ & $0.346$ \\
sink+proj & $0.378$ & $0.201$ & $0.168$ & $0.360$ & $0.194$ & $0.201$ & $0.321$ & $0.166$ & $0.232$ \\
\bottomrule
\end{tabular}

    \end{small}
  \end{center}
\end{table}

\begin{table}[h]
  \caption{Zero-shot evaluation of the 124M models, mean and standard
  deviation over three seeds, each model evaluated at its 1{,}024-token
  context with the GPT-2 tokenizer. LAMBADA is last-word accuracy on the
  5{,}153 passages of the OpenAI variant of the test set: the final word
  of each passage is the target, and a passage counts as correct only if
  every token of the target is the greedy prediction at its position.
  HellaSwag is accuracy on the first 3{,}000 validation examples,
  choosing among the four endings by the mean per-token log-probability
  of the ending given the context. WikiText-2 is perplexity on the raw
  validation split, concatenated and cut into non-overlapping
  1{,}024-token windows, computed as the exponential of the mean
  next-token loss.}
  \label{tab:downstream}
  \begin{center}
    \begin{small}
      \begin{tabular}{lrrr}
\toprule
Variant & LAMBADA acc & HellaSwag acc & WikiText-2 ppl \\
\midrule
baseline & $0.182 \pm 0.006$ & $0.323 \pm 0.002$ & $58.1 \pm 0.5$ \\
sink logit & $0.179 \pm 0.004$ & $0.325 \pm 0.003$ & $57.0 \pm 0.5$ \\
norm gate & $0.185 \pm 0.001$ & $0.323 \pm 0.002$ & $57.3 \pm 0.2$ \\
projection gate & $0.186 \pm 0.002$ & $0.325 \pm 0.003$ & $57.3 \pm 0.8$ \\
sink+norm & $0.186 \pm 0.003$ & $0.323 \pm 0.002$ & $56.8 \pm 0.1$ \\
sink+proj & $0.187 \pm 0.007$ & $0.326 \pm 0.003$ & $56.6 \pm 0.3$ \\
\bottomrule
\end{tabular}

    \end{small}
  \end{center}
\end{table}

\section{Synthetic Superposition Task}\label{sec:habitat}

A synthetic task places a single attention head under controlled
superposition load, with features superposed in the value space at loads
from one to eight features per dimension. Routing in this task is
channel-blind by construction, meaning that keys carry only tags and the
query cannot see the content of any read. \Cref{fig:toy} shows three
things. The baseline's error grows rapidly with load. Off-by-one
abstention does not help. The norm gate approaches the performance of an
oracle trained with knowledge of the interference. A condition in which
routing can see content removes the gate's advantage.

\paragraph{Construction.} A dictionary of $N$ unit
feature vectors in $\R^d$ with $d = 64$ is drawn at random. When $N \le
d$ the dictionary is orthonormalized so that features do not
interfere, which serves as the no-interference control. When $N > d$ the
features are in superposition at load $N/d$, which we vary over 1, 2, 4,
and 8. The first $n_{\text{rel}} = 16$ features
are designated relevant. Each sequence contains $T = 64$ items, and
each item is a relevant feature with probability 0.15 and an
irrelevant feature otherwise. A single learned query attends over the
items through one head with no
feed-forward network and no layer normalization. A linear readout of the
head's output must report which relevant features were present, trained
with a multi-label binary cross-entropy loss. Models are trained for
4{,}000 steps with
batches of 256 sequences at a learning rate of $10^{-3}$, with three
seeds per configuration.

The keys are the design variable. In the channel-blind condition used in
\cref{fig:toy}, each item's key
is a random tag vector that carries no information about the item's
content, so routing cannot distinguish relevant from irrelevant items
and must spread mass over both. The value pathway is then the only place
where interference from irrelevant items can be removed. The variants
are the baseline,
off-by-one abstention, and the norm gate of \cref{sec:armlist}, each with the gate sharpness $\beta = 8$ of the main experiments. The
oracle is the baseline architecture trained and evaluated with every
irrelevant item's read set to zero by ground truth, which gives the
error attainable if filtering were perfect. In the routing-visible
condition the keys are
the items themselves, so that a learned key direction can separate
relevant from irrelevant items. In that condition routing places
almost all of its mass on relevant items, the oracle and the baseline
coincide, and no variant has any filtering headroom. The task therefore has a filtering signature only when routing is blind
to content.

\begin{figure}[h]
  \begin{center}
    \centerline{\includegraphics[width=0.5\textwidth]{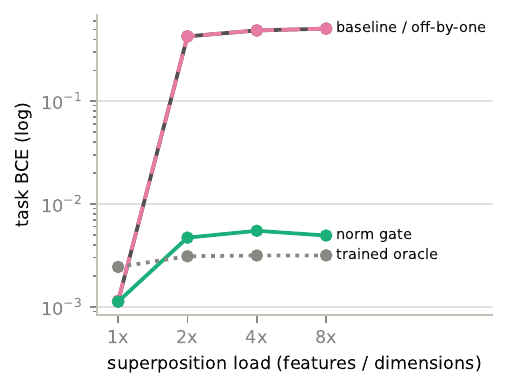}}
    \caption{The channel-blind superposition task. Task error against
    superposition load for the baseline, off-by-one abstention, the
    norm gate, and a trained oracle. The norm gate approaches the oracle's
    error, whereas off-by-one abstention does not help.}
    \label{fig:toy}
  \end{center}
\end{figure}

\section{Pretrained Models}\label{app:incidence}

\paragraph{Injection in pretrained models.}
The injection experiment of \cref{sec:fingerprint} can be run on any
pretrained model, because it modifies only the value reads at
evaluation time. \Cref{tab:pythiainject} reports it for six Pythia models from 160M to
2.8B parameters \citep{biderman2023pythia}, on WikiText-2 blocks of
1{,}024 tokens \citep{merity2017wikitext}, with
the quietest quarter of each head's reads corrupted and the junk
normalized to the head's median read norm as in the main experiments.
Two features of our results replicate at every size from 160M to
2.8B. Structured junk transplanted from another context costs one and a half
to three times as much as Gaussian noise of the same norm at doses up to
0.8. The two converge only at the highest dose, where every condition is
destructive. Junk transplanted from the same
sequence costs as much as or more than junk from another context.
The absolute damage at a given dose is larger than in our models. The
comparison the replication supports is between junk types within a
model.

\begin{table}[h]
  \caption{Injection replication on pretrained Pythia models. Increase
  in WikiText-2 NLL (nats) over the clean model when the quietest
  quarter of each head's reads receives junk of the head's median read
  norm scaled by $\epsilon$, transplanted from another context,
  transplanted from the same sequence, or drawn as Gaussian noise.}
  \label{tab:pythiainject}
  \begin{center}
    \begin{small}
      \begin{tabular}{llrrrr}
\toprule
Model & Junk & $\epsilon=0.2$ & $\epsilon=0.4$ & $\epsilon=0.8$ & $\epsilon=1.6$ \\
\midrule
Pythia-160M & other context & $+0.13$ & $+0.52$ & $+2.24$ & $+5.35$ \\
 & same context & $+0.12$ & $+0.51$ & $+2.42$ & $+5.43$ \\
 & Gaussian & $+0.08$ & $+0.25$ & $+1.18$ & $+4.78$ \\
\midrule
Pythia-410M & other context & $+0.15$ & $+0.94$ & $+5.47$ & $+7.89$ \\
 & same context & $+0.16$ & $+1.09$ & $+6.08$ & $+7.79$ \\
 & Gaussian & $+0.09$ & $+0.43$ & $+2.31$ & $+6.72$ \\
\midrule
Pythia-1B & other context & $+0.11$ & $+0.71$ & $+3.59$ & $+8.99$ \\
 & same context & $+0.11$ & $+0.77$ & $+4.19$ & $+9.77$ \\
 & Gaussian & $+0.07$ & $+0.33$ & $+1.63$ & $+5.46$ \\
\midrule
Pythia-1.4B & other context & $+0.11$ & $+0.79$ & $+4.28$ & $+6.93$ \\
 & same context & $+0.12$ & $+0.79$ & $+4.50$ & $+7.33$ \\
 & Gaussian & $+0.08$ & $+0.41$ & $+2.58$ & $+7.22$ \\
\midrule
Pythia-2.8B & other context & $+0.11$ & $+0.78$ & $+4.85$ & $+7.86$ \\
 & same context & $+0.11$ & $+0.80$ & $+4.78$ & $+8.22$ \\
 & Gaussian & $+0.05$ & $+0.26$ & $+2.02$ & $+7.43$ \\
\bottomrule
\end{tabular}

    \end{small}
  \end{center}
\end{table}

\paragraph{A production model with a sink logit.}
By the account of \cref{sec:autopsy}, a model trained with a learned
per-head sink logit has no need to construct a sink of its own. It
should therefore place a large share of its attention mass on the
phantom, park little on the first key position, and keep an ordinary
value read at that position. We check this on
gpt-oss-20b \citep{openai2025gptoss}, whose released weights carry
the sink logit of every head, on eight WikiText-2 blocks of 1{,}024
tokens presented as raw text. In the Hugging Face implementation the
sink column is dropped after the softmax, so a head's phantom mass on
a query is one minus the sum of its returned attention weights.
\Cref{tab:gptoss} reports the means over layers and heads. The twelve
sliding-window layers place 48\% of their mass on the phantom,
0.3\% on position 0, and read position 0 at an ordinary norm (ratio
0.98). The twelve full-attention layers place 34\% on the phantom. Six of them
still put 9\% to 17\% of their mass on position 0, but with a position-0
read norm at or above the typical norm (ratios of 1.1 to 2.9), so that
position is read as content rather than crushed into a sink. The learned
sink logits average between 1.1 and 4.1 per
layer and are larger in the full-attention layers. In all three
respects the model behaves as our sink-logit models do, at about
sixty times the size of the largest we train.

\begin{table}[h]
  \caption{Sink measurement on gpt-oss-20b, mean over heads and
  layers of each type: attention mass on the phantom, share
  of the remaining mass on the first key position, and the
  position-0 value-norm ratio.}
  \label{tab:gptoss}
  \begin{center}
    \begin{small}
      \begin{tabular}{lrrr}
\toprule
Layer type & phantom mass & mass on position 0 & $r_0$ ratio \\
\midrule
sliding-window layers (12) & $0.48$ & $0.003$ & $0.98$ \\
full-attention layers (12) & $0.34$ & $0.064$ & $1.31$ \\
\bottomrule
\end{tabular}

    \end{small}
  \end{center}
\end{table}

\section{Mechanism Telemetry}\label{app:telemetry}

\paragraph{Sink analysis.}
\Cref{fig:autopsy} reports the measurement discussed in
\cref{sec:autopsy}. For each 124M model at seed 0, attention weights are
recomputed explicitly on fixed validation batches, and the figure reports the attention mass received by the first key
position and by the phantom where one exists, together with the norm
ratio of the value read at position 0. \Cref{tab:sink} lists the values
at 124M and at 350M.

\begin{table}[h]
  \caption{Sink analysis at the
training context (1024 tokens), seed 0: attention mass on the first key
position as a percentage of the mass on real keys, mass on the phantom
as a percentage of all mass, and the norm ratio of the position-0 read.
The phantom share is zero by construction for variants without one, and
dashes mark variants not trained at 350M.}
\label{tab:sink}
  \begin{center}
    \begin{small}
      \begin{tabular}{lrrrrrr}
\toprule
 & \multicolumn{3}{c}{124M} & \multicolumn{3}{c}{350M} \\
\cmidrule(lr){2-4} \cmidrule(lr){5-7}
Variant & pos. 0 (\%) & phantom (\%) & norm ratio & pos. 0 (\%) & phantom (\%) & norm ratio \\
\midrule
baseline & $6.0$ & $0$ & $0.42$ & $11.1$ & $0$ & $0.36$ \\
norm gate & $3.7$ & $0$ & $0.65$ & -- & -- & -- \\
projection gate & $7.1$ & $0$ & $1.16$ & -- & -- & -- \\
sink logit & $0.5$ & $31.1$ & $0.99$ & $0.7$ & $53.0$ & $1.01$ \\
sink+norm & $0.4$ & $18.6$ & $0.96$ & $0.7$ & $34.6$ & $0.99$ \\
sink+proj & $0.8$ & $20.1$ & $0.99$ & $1.2$ & $23.2$ & $1.04$ \\
\bottomrule
\end{tabular}

    \end{small}
  \end{center}
\end{table}

\paragraph{Selectivity.}
For a trained norm gate, the attention-weighted mean of $g_j$ over the
reads a query attends to measures how much the gate attenuates on
average. The attention-weighted standard deviation of $g_j$ within the
row measures how much it discriminates among reads. A
gate acting purely as an abstention device would have a large between-row spread of the mean and a selectivity of
zero. At seed 0 the selectivity of the norm gate is 0.054, 0.170, and
0.177 at 10M, 50M, and 124M (the norm gate was not run at 350M). That of
the sink+norm variant's gate is 0.025, 0.138, 0.158, and 0.150 at the
four tiers. The same quantity is defined for the
projection gate, and it is larger: 0.092, 0.218, and 0.228 for the
projection gate alone at the first three tiers, and 0.050, 0.217,
0.228, and 0.204 for sink+proj at the four tiers. \Cref{fig:trend}(d) plots the norm gate, sink+norm, and sink+proj
series. The attention-weighted mean
gate at 124M is 0.46 for the norm gate, 0.58
for sink+norm, 0.47 for the projection gate, and 0.52 for sink+proj.
Both gate forms attenuate substantially on average, and each attenuates
less when a phantom is present.

\paragraph{What the filter removes.}
Averaging the gate value assigned to each token over its occurrences
in the validation set gives a picture of what the filter removes. At 124M the most strongly attenuated tokens for the norm gate and the
projection gate are the
newline, the period, the comma, the semicolon, the opening
parenthesis, and the articles and conjunctions (\emph{the},
\emph{The}, \emph{a}, \emph{and}), with mean gate values between 0.36
and 0.43 for the norm gate. The least attenuated tokens are content
carriers such as \emph{which}, \emph{not}, \emph{can}, \emph{will},
\emph{more}, and \emph{you}, with mean gate values between 0.50 and
0.56. The filter therefore attenuates the reads of high-frequency structural
tokens most strongly.

\paragraph{Token identity and content.}
The token table invites the objection that the filter has learned a
prior over token identity rather than a content-dependent rule. We collect every gate value together with the token at its position over
262k validation tokens and compute the fraction of the variance of the
gate values that token identity explains, pooling the between-token and
within-token sums of squares over all heads of a model. Token
identity explains 81\% of the variance for the norm
gate, 77\% for the projection gate, and 81\% and 77\% for sink+norm
and sink+proj, or 70\% to 75\% when only tokens seen at least fifty
times are counted. The share is highest in the first three layers,
where the residual stream is still close to the token embedding
(0.87 to 0.96 for the norm gate), lowest in the fourth and fifth
layers (0.4 to 0.6 for every variant), and between 0.7 and 0.9 in the
later layers. On natural text the trained filter is therefore largely, but not
entirely, a function of token identity. The injection experiments of \cref{sec:fingerprint} show that the gate
also acts on content, since the junk they remove changes a read's norm
or direction without changing its token.

\end{document}